\documentclass[letterpaper]{article}
\usepackage[preprint]{preprintfmt}
\usepackage[hyphens]{url}
\usepackage{graphicx}
\usepackage{natbib}
\usepackage{caption}
\usepackage{float}
\usepackage{booktabs}
\usepackage{multirow}
\usepackage{amsmath,amssymb,amsthm}
\usepackage{xspace}

\newtheorem{proposition}{Proposition}
\newtheorem{definition}{Definition}
\newtheorem{assumption}{Assumption}

\newcommand{\ret}{\mathrm{ret}}
\newcommand{\safe}{\mathrm{safe}}
\newcommand{\mem}{M}
\newcommand{\query}{q}
\newcommand{\gain}{\Delta}
\newcommand{\feat}{\psi}

\title{Retain or Consolidate? Budget-Dependent Operator Selection for Language Agent Memory}
\author{Qingcan Kang$^{1*}$, Mingyang Liu$^{2*}$, Shixiong Kai$^{1}$,
Kaichao Liang$^{1}$,\\ Zhentao Tang$^{1}$, Yuqi Cui$^{1}$,
Tao Zhong$^{1}$, Mingxuan Yuan$^{1\dagger}$}
\affiliations{$^{1}$Huawei Noah's Ark Lab\\
$^{2}$Department of Computer Science, City University of Hong Kong\\
\texttt{\{kangqingcan, kaishixiong, liangkaichao, tangzhentao1\}@huawei.com}\\
\texttt{\{cuiyuqi1, zhongtao5, Yuan.Mingxuan\}@huawei.com}\\
\texttt{mingyaliu8-c@my.cityu.edu.hk}\\
$^{*}$Equal contribution. $^{\dagger}$Corresponding author.}

\begin{document}

\maketitle

\begin{abstract}
Memory allows language agents to draw on information from past interactions to support future reasoning and action. However, the finite context windows and inference costs of large language models (LLMs) limit how much stored information can be used. Existing agent systems primarily address this challenge through two memory-management strategies: retention and consolidation. Retention preserves exact details in raw records but may exclude relevant evidence under tight budgets; consolidation improves evidence coverage through compression but risks losing query-critical details. Neither strategy is universally preferable. This raises two coupled questions: \emph{when} should consolidation replace retention, and \emph{which} operator---\mbox{\textsc{Merge}}, \mbox{\textsc{Abstract}}, or \mbox{\textsc{Rewrite}}---should be used? We formalize the joint decision by decomposing each consolidation operator's value into a coverage effect on evidence omitted by retention and a signed replacement effect from replacing raw evidence that already fits. This balance explains why the preferred action changes with relative budget pressure. We operationalize the mechanism with Offline Abstraction-Safety (OAS), a lightweight learner that estimates action utilities from pre-generation features and uses held-out harm calibration to limit harmful replacements. Experiments on the public LongMemEval and LoCoMo benchmarks reveal a consistent budget-dependent pattern. On LongMemEval, consolidation improves accuracy by up to 48 percentage points under tight budgets, whereas retention is preferable under loose budgets, when most relevant raw evidence can be retained in full. LoCoMo reproduces this crossover at a smaller absolute budget, consistent with its shorter evidence. These results suggest that the choice between retention and consolidation depends on evidence length relative to the available budget, rather than on a fixed token threshold. Across both benchmarks, cross-note \textsc{Abstract} and \textsc{Merge} generally outperform local \textsc{Rewrite} when compression is necessary.

\end{abstract}

\section{Introduction}
\label{sec:intro}

Language agents built on large language models (LLMs) often operate beyond a
single prompt. They accumulate observations, dialogue, and retrieved facts to
support later decisions, but this history cannot grow without consequence. An
LLM has a finite context window, and longer inputs also increase inference cost
and latency before that limit is reached; relevant evidence can also be used
unreliably within long contexts \citep{liu2024lost}. A token budget is therefore a practical
deployment constraint. The memory system must fit useful evidence into that
budget without unnecessarily degrading the answer.

Existing systems commonly manage this constraint through two strategies:
\emph{memory retention} and \emph{memory consolidation}. Retention selects raw
records and leaves their contents unchanged. It preserves exact wording,
provenance, timestamps, and corrections, but uses tokens inefficiently when
evidence is redundant or distributed; under a tight budget, a relevant record may
not fit. Prior systems mix raw-record storage with hierarchy, forgetting,
summarization, or reflection
\citep{packer2023memgpt,zhong2024memorybank,park2023generative}.
\citet{oslmr} further formulate long-horizon retention as constrained stochastic
optimization and learn which raw notes to preserve under delayed miss,
reacquisition, and staleness costs. Consolidation instead generates a compact
replacement for one or more records. Summarizing, merging, or rewriting can remove
redundancy and increase evidence coverage per token, but generation may omit a
decisive detail, blur temporal order, discard a correction, or introduce
unsupported content. Mem0 and A-MEM generate structured memories, while RecMem
and TrustMem study recurrence-triggered consolidation and reliable transitions
\citep{chhikara2025mem0,xu2025amem,dai2026recmem,yang2026trustmem}.

These research lines establish both strategies, but they do not resolve their
joint decision. Retention research asks which original records to keep, whereas
consolidation research largely asks how to construct compact memories. It remains
unclear when generated compression should replace raw evidence and which
consolidation operator should be used. The answer depends on budget pressure. If
relevant raw evidence does not fit, consolidation may recover otherwise omitted
information. If that evidence already fits, replacing it may sacrifice exact
details without adding useful coverage. This tradeoff motivates our central
questions: \emph{when} should an agent consolidate, and \emph{which} operator---
\textsc{Merge}, \textsc{Abstract}, or \textsc{Rewrite}---should it choose?
We isolate this comparison by holding the query, candidate evidence, retriever,
answer model, and answer-time budget fixed while changing only the memory
representation.

We formulate retention and the three consolidation operators as a finite,
budget-dependent action set. Each action has a well-defined conditional expected
answer utility. The optimal action would be obtained by comparing these four
utilities, but the utilities of unexecuted actions are counterfactual and hence
unavailable at deployment; the obstacle is missing utility information, not a
difficult combinatorial solver. To explain how the preferred action changes, we
introduce an idealized mechanism model that decomposes each operator's value into
two terms: a \emph{coverage effect} on evidence omitted by retention and a signed
\emph{replacement effect} on raw evidence that already fits. Their balance changes
with relative budget pressure and determines both when consolidation is beneficial
and which operator provides the greatest expected utility.

We operationalize this formulation with Offline Abstraction-Safety (OAS), a
lightweight multi-action utility learner. OAS estimates the unavailable action
utilities from features observed before generation, including budget scale,
evidence-fit pressure, cluster geometry, and query type, and then applies a plug-in
maximizer. A held-out safety threshold limits harmful replacements. Linear
regression is the primary estimator because the available supervision favors a
low-capacity model. A multilayer perceptron (MLP) serves as a nonlinear capacity
control, testing whether ridge regression underfits the action-utility relationship.
This design uses learning where it is needed:
to estimate counterfactual utilities rather than to obscure a simple finite-action
optimization.

Experiments on the public LongMemEval and LoCoMo benchmarks support this account.
On LongMemEval, consolidation improves absolute accuracy over retention by as much
as $48\%$ under tight budgets, whereas retention leads every consolidation
operator by $8\%$--$11\%$ once the raw evidence largely fits. A second answer
model reproduces this reversal. LoCoMo exhibits the same pattern, but its crossover
occurs at a smaller absolute budget because its evidence is shorter. This result
supports relative budget pressure rather than a universal token threshold.
Operator comparisons address the \emph{which} question: cross-note
\textsc{Abstract} and \textsc{Merge} are generally more effective than local
\textsc{Rewrite} under compression, although neither cross-note operator dominates
in every setting. OAS improves the controlled when--which decision under
budget-stressed conditions. Its separately evaluated full-history variant does not
surpass the strongest fixed policy, which bounds our claim: the experiments
validate the budget-dependent mechanism more strongly than general-purpose
instance-level routing.

Our contributions are as follows.
\begin{itemize}
  \item \textbf{A unified when--which formulation.} We place retention and three
  generative consolidation operators in the same decision space with signed,
  budget-dependent utility, making the decision of whether and how to consolidate
  explicit in the problem formulation.
  \item \textbf{A testable budget-dependent mechanism.} An idealized surrogate
  decomposes consolidation value into a signed coverage effect and a replacement
  effect. Experiments on public benchmarks identify the predicted reversal
  and show that operator choice depends on budget pressure and evidence structure
  in the mechanism analysis and experiments.
  \item \textbf{A lightweight learner that operationalizes the formulation.} OAS
  learns multi-action utility from pre-generation features and uses held-out harm
  calibration. Controlled and independent evaluations distinguish successful
  budget-aware selection from the present boundary of transferable routing
  as detailed in the methodology and supplementary material.
\end{itemize}

\begin{figure}[t]
\centering
\includegraphics[width=0.98\linewidth]{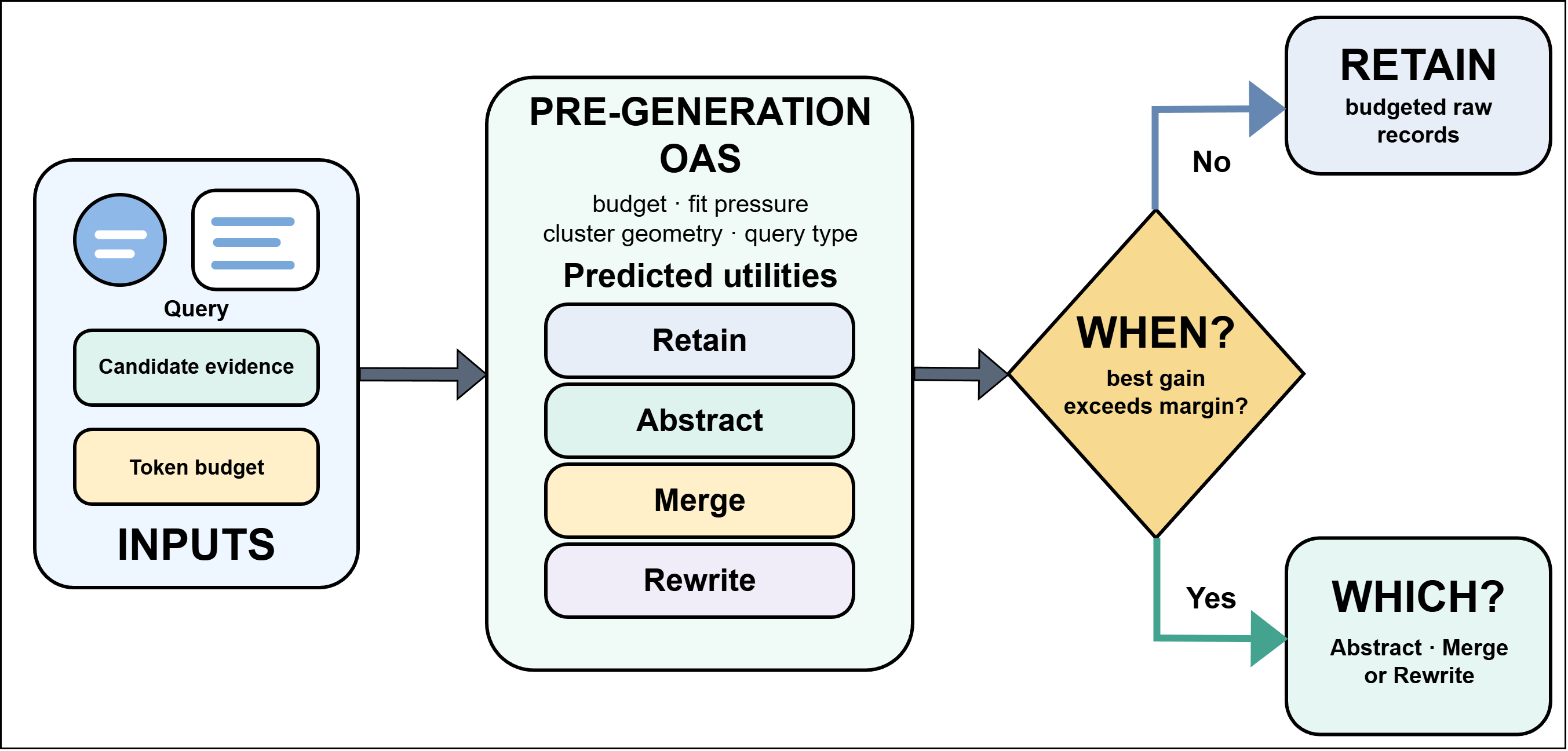}
\caption{\textbf{OAS workflow for deciding when to consolidate and which operator to use.}
From pre-generation features, OAS estimates retention and three consolidation
utilities. It consolidates when the best predicted gain exceeds a margin---zero
for Direct OAS and calibrated on held-out questions for Calibrated OAS---and
otherwise retains (\emph{when}); it then applies the highest-utility generative
operator (\emph{which}).}
\label{fig:compact-concept}
\end{figure}
 \section{Related Work}
\label{sec:related}

\paragraph{Agent memory systems and consolidation.}
Surveys of agent memory review representations, evaluation, and operations such
as consolidation, forgetting, retrieval, and compression
\citep{zhang2024memorysurvey,du2025rethinking}. Existing systems combine storage
and retrieval with hierarchy, reflection, summarization, or structured updates
\citep{packer2023memgpt,zhong2024memorybank,park2023generative}. More recent work
studies consolidation, linked notes, recurrence-based triggers, or transition
verification
\citep{chhikara2025mem0,xu2025amem,dai2026recmem,yang2026trustmem}.
These systems design pipelines or individual operators. We instead hold the
query, evidence, and answer-time budget fixed to isolate when transformation is
beneficial and which operator should perform it.

\paragraph{Retrieved-context and prompt compression.}
RECOMP learns extractive and abstractive compressors, while LLMLingua controls
prompt length through token-level compression
\citep{xu2023recomp,jiang2023llmlingua}. Our intervention instead replaces
episodic evidence with a generated memory and compares retention with multiple
semantic operators on matched instances. The objective is answer utility, not
compression rate: it determines \emph{when} to transform memory and \emph{which}
transformation to select.

\paragraph{Conversational-memory benchmarks.}
LongMemEval tests five memory abilities, including temporal reasoning and
knowledge updates; LoCoMo evaluates long-session question answering, event
summarization, and multimodal dialogue generation
\citep{wu2025longmemeval,maharana2024locomo}. Exact details matter for some
questions, whereas distributed evidence creates compression pressure, making both
benchmarks natural tests of our when--which mechanism.

\paragraph{Selectional versus generative memory.}
\citet{oslmr} formulate long-horizon retention as constrained sequential
optimization and learn which unaltered notes to keep. Our complementary decision
asks whether identified evidence should remain raw or be replaced by a generated
record. Classical budgeted selection has approximation guarantees under fixed
coverage or submodular objectives \citep{nemhauser1978,khuller1999}, but generation
can either improve coverage or reduce fidelity. We model this signed tradeoff with
operator-specific utilities and learn the resulting finite-action decision,
following the broader learning-assisted-algorithms principle
\citep{mitzenmacher2021}.
 \section{Methodology}
\label{sec:method}

We formalize memory consolidation as a one-step query-time decision under a token
budget imposed by context capacity or deployment cost. The objective is to use the
available context efficiently while maintaining answer utility. We formalize why
consolidation can either improve or reduce answer utility, and then
learn a lightweight utility router that decides when to consolidate and which operator to
apply from features available before generation. The presentation has three steps:
we first define the four-action decision, then derive an idealized mechanism for
its budget dependence, and finally describe how OAS learns the decision from
offline outcomes without observing generated records at test time.

\subsection{Problem Formulation}
\label{sec:decision}

\paragraph{Memory, budget, and retrieval.}
Let $\mem=\{m_i\}_{i=1}^N$ be an agent memory whose notes contain text, metadata,
and embeddings $\mathbf{e}_i$. A retriever $R_B$ maps a query $\query$ and a memory state to an
ordered subset $S=R_B(\query,\mem)$ whose token cost obeys the context budget
$B$; that is, the total cost of the selected notes does not exceed the budget.
The budget may reflect the model's context-window limit or a smaller operating
limit chosen to control inference cost and latency. An answer model then produces
a response from $(\query,S)$, which a fixed judge grades against the reference
answer to obtain the binary utility
$u(\query,S)\in\{0,1\}$. Within each paired comparison, retriever, answerer,
grader, and budget are fixed, so only the memory representation changes.

\paragraph{Generative consolidation operators.}
For a candidate evidence cluster $C\subseteq\mem$, the agent may apply a
\emph{generative} operator that emits a budget-targeted representation
$\tilde m_{o,B}(C)$ not present among the raw records. The generative family is
$\mathcal O_{\rm gen}=\{\textsc{Merge},\textsc{Abstract},\textsc{Rewrite}\}$.
\textsc{Merge} consolidates several notes into one compact factual record,
\textsc{Abstract} synthesizes a higher-level note across its inputs, and
\textsc{Rewrite} reformulates each note independently before packing the rewritten
notes under the same budget. For notation, the packed Rewrite output is treated as
one representation $\tilde m_{o,B}(C)$ while preserving its internal note boundaries.
These are distinct from \emph{selectional} operators that retain, discard, or
re-index existing raw content. Each generative operator in our study
\emph{replaces} the packed constituents in the answer context (freeing budget), yielding a
post-consolidation state versus the retain-as-is state $\mem$.

All three operators can change both coverage and fidelity, so each requires a
selection decision. We write $T_{C,o,B}(\mem)$ for replacing cluster $C$ with the
representation generated by operator $o\in\mathcal O_{\rm gen}$. The controlled experiment
studies a \emph{query-time packing} decision: after a query identifies a candidate
evidence cluster but before a compact record is generated, the policy chooses
retention or one generative operator. This scope makes query type observable and
does not claim to solve query-agnostic background consolidation. We analyze one
query--cluster decision at a time; interactions among repeated memory updates are
outside the present guarantee.

\begin{definition}[Consolidation gain]
\label{def:gain}
For a memory state $\mem$, candidate cluster $C$, query $\query$, budget $B$, and
operator $o$, the realized paired consolidation gain is
\begin{equation}
\label{eq:gain}
\begin{aligned}
  \gain_o(z)&=u_o(z)-u_{\ret}(z),\\
  u_o(z)&=u\!\left(\query,R_B(\query,T_{C,o,B}(\mem))\right).
\end{aligned}
\end{equation}
the change in graded answer utility from consolidating $C$ rather than retaining
it, under the same retriever, budget, answerer, and grader. This is a paired
run-level quantity and therefore takes values minus one, zero, or one. Generation,
answering, and grading randomness are averaged when we define expected action
utilities below.
\end{definition}

\paragraph{The decision problem.}
A \emph{packing policy} $\pi$ maps an observable instance
$z=(\mem,C,\query,B)$ to an action in
$\mathcal A=\{\ret\}\cup\mathcal O_{\rm gen}$. Let
$u_{\ret}(z)=u(\query,R_B(\query,\mem))$ and
$u_o(z)=u(\query,R_B(\query,T_{C,o,B}(\mem)))$ for a generative action $o$.
The value of a policy is simply the expected utility of its selected action,
where the expectation covers instances, generation, answering, and grading:
\begin{equation}
\label{eq:policy-value}
  V(\pi)=\mathbb E\big[u_{\pi(z)}(z)\big].
\end{equation}
Let $\pi_{\ret}$ denote the fixed policy that always retains the raw records.

\paragraph{Well-definedness and scope.}
For each instance, the candidate cluster $C$ is supplied by an upstream retriever
or evidence selector and is treated as fixed by the packing decision. Every action
produces a context whose token cost is at most $B$: retention packs complete raw
records, while each generated representation is capped at the same budget. The
realized action utilities are bounded binary random variables, so every conditional
mean $\mu_a(z)$ below exists. Because the action set is finite, at least one
utility-maximizing action always exists; ties are resolved toward retention. The
formulation therefore defines a valid one-step decision for any supplied
$(\mem,C,\query,B)$. It deliberately does not optimize candidate discovery or repeated
memory updates, which are upstream and sequential extensions, respectively.

\paragraph{The joint ``when'' and ``which'' decision.}
Let $\mu_a(z)=\mathbb E[u_a(z)\mid z]$ denote the conditional utility of action
$a\in\mathcal A$, and define the best generative advantage
\begin{equation}
\label{eq:joint-decision}
\begin{aligned}
  \Gamma(z)&=\max_{o\in\mathcal O_{\rm gen}}\mu_o(z)-\mu_{\ret}(z),\\
  o^\star(z)&\in\arg\max_{o\in\mathcal O_{\rm gen}}\mu_o(z).
\end{aligned}
\end{equation}
The two questions are therefore one decision: \emph{when} means consolidating
when $\Gamma(z)$ is positive (or exceeds a safety margin), and \emph{which} means applying
$o^\star(z)$ when consolidation is selected. This formulation does not assume that an
operator has a fixed global rank: both comparisons depend on the query, candidate
cluster, and budget contained in $z$.  OAS below estimates these action utilities
from pre-generation features.
Equivalently, with ties resolved toward retention, the optimal action is
\begin{equation}
\label{eq:optimal-action}
  a^\star(z)=
  \begin{cases}
    \ret, & \text{if }\Gamma(z)\leq 0,\\
    o^\star(z), & \text{if }\Gamma(z)>0.
  \end{cases}
\end{equation}

\paragraph{Why learning is required.}
If all four conditional utilities $\{\mu_a(z):a\in\mathcal A\}$ were known, the
problem would be solved exactly by comparing four numbers; no general-purpose
solver would be required. The deployment difficulty is statistical rather than
combinatorial. Before a decision, the generated records do not yet exist, and
their answer utilities are counterfactual and unknown; reference-based correctness
is also unavailable in normal deployment. Generating and scoring every alternative
would defeat low-cost routing, and a numerical solver cannot recover missing
objective values. We therefore evaluate all actions offline, learn their utilities
from pre-generation features, and use the estimates in the same four-action
maximization at deployment. OAS approximates the unknown objective, not a
computationally difficult argmax.

\subsection{An Idealized Mechanism for Budget Dependence}
\label{sec:mechanism}

We use a deliberately simplified proxy utility to isolate two effects of replacing
raw memories with a generated record. Here, \emph{surrogate} means an analyzable
mechanism model: it is designed to explain the direction of the observed accuracy
change, but is not assumed to equal the LLM-judged accuracy $u$ in
Definition~\ref{def:gain}. The construction applies separately to every generative
operator; we omit the operator subscript temporarily for readability. This model
describes the average budget effect over a query population. The instance-level
decision remains the conditional-utility problem in
Eq.~\eqref{eq:joint-decision} and is learned by OAS below.

Let $\mathcal Q$ denote the set of possible queries and let $q\sim F$ denote a query drawn
from the evaluation distribution. The query is the only random variable in this
surrogate; the memory set, budget, retriever, and generated record are held fixed.
The distribution $F$ represents how frequently queries occur; $\Pr_F$ and
$\mathbb E_F$ denote probability and expectation under $F$. We do not fit a
parametric density: in our benchmark analysis, $F$ is the uniform empirical
distribution over evaluation questions, so probabilities are question fractions
and expectations are question averages. For any note $m$ and query it supports,
the signed fidelity $\varphi_m(q)\in[-1,1]$ represents whether the note provides
accurate evidence or misleading evidence: positive values denote correct support,
negative values denote misleading support, and zero denotes no contribution. In this deliberately simplified model,
$s_B(q,M)$ is the signed fidelity contributed by the highest-ranked applicable
note retrieved from memory state $M$ under budget $B$; it is zero if no retrieved
note supports the query. We define the proxy
utility as the expected contribution
\begin{equation}
  \mathcal U_B(M)=\mathbb E_F[s_B(q,M)].
\end{equation}
This expectation is well defined because $s_B$ is bounded between minus one and
one. On a finite benchmark, it is simply the average contribution across queries.

For the fixed budget under analysis, let $H_B\subseteq\mathcal Q$ be the queries
supported by the generated record $\tilde m_B$, and assume that $H_B$ has positive
probability under $F$. Let $G_B\subseteq\mathcal Q$ be the raw-grounded region:
the queries for which retention retrieves an applicable constituent of $C$.
Define $X_B=H_B\setminus G_B$ as the queries supported by the generated record but
missed by retention. Under nested raw packing, increasing the budget can only
enlarge $G_B$; the generated coverage $H_B$ may also vary with the target budget.

\begin{assumption}[Localized retrieval change]
\label{ass:uncovered}
Replacing $C$ by $\tilde m_B$ leaves the retrieved contribution unchanged outside
$H_B$; on $X_B$ retention contributes zero and consolidation retrieves $\tilde m_B$.
Within $H_B\cap G_B$, let $J_B$ be the replacement region on which retrieval
changes from an incumbent raw note $m_{\rm raw}(q)$ to $\tilde m_B$.
\end{assumption}

This assumption localizes the intervention: replacing $C$ may change evidence
only for queries supported by the generated record. It is used solely to obtain
the exact two-term decomposition in Proposition~\ref{prop:decomp}; OAS and the
empirical comparisons do not require it. The assumption is restrictive because it
excludes retrieval or packing spillovers outside $H_B$. Without it, the utility
change contains an additional residual spillover term.

\begin{proposition}[Operator-indexed budget decomposition]
\label{prop:decomp}
For each $o\in\mathcal O_{\rm gen}$, let $H_{B,o}$ be the queries supported by
$\tilde m_{o,B}$, let $X_{B,o}=H_{B,o}\setminus G_B$, and let
$J_{B,o}\subseteq H_{B,o}\cap G_B$ be its replacement region. Under
Assumption~\ref{ass:uncovered}, the operator's
surrogate-utility change
$\Delta\mathcal U_{B,o}=\mathcal U_B(T_{C,o,B}(\mem))-\mathcal U_B(\mem)$ is
\begin{equation}
\label{eq:decomp}
\begin{aligned}
\Delta\mathcal U_{B,o}&=C_{B,o}+R_{B,o},\\
 C_{B,o}&=\Pr_F(q\in X_{B,o})\,
 \mathbb E_F[\varphi_{\tilde m_{o,B}}(q)\mid q\in X_{B,o}],\\
 R_{B,o}&=\Pr_F(q\in J_{B,o})\,
 \mathbb E_F[\varphi_{\tilde m_{o,B}}(q)\\
&\hspace{45pt}-\varphi_{m_{\rm raw}(q)}(q)\mid q\in J_{B,o}].
\end{aligned}
\end{equation}
Here $C_{B,o}$ and $R_{B,o}$ are the coverage and replacement effects,
respectively.
Each term is defined as zero when its conditioning event has probability zero.
The coverage effect is positive when $X_{B,o}$ has positive probability and the
generated record has positive mean fidelity there. The replacement effect is
nonpositive if the generated record has no greater fidelity than the incumbent raw
record throughout $J_{B,o}$.
\end{proposition}

The proof partitions queries into $X_{B,o}$, $J_{B,o}$, and the unchanged
remainder and then takes expectations; full details are in the supplementary
material.

Proposition~\ref{prop:decomp} gives the value of each generative action separately.
The next result converts these operator-indexed gains into the joint
\emph{when}--\emph{which} decision.

\begin{proposition}[Joint rule for when and which]
\label{prop:joint}
Let $\Delta\mathcal U_{B,o}$ be the operator-specific value from
Proposition~\ref{prop:decomp}. For a fixed cluster, budget, and query distribution,
assign retention gain zero and let
$o^\star\in\arg\max_{o\in\mathcal O_{\rm gen}}\Delta\mathcal U_{B,o}$.
The optimal surrogate action is retention when
$\Delta\mathcal U_{B,o^\star}\leq 0$, and $o^\star$ otherwise. Therefore,
\emph{when} asks whether the best consolidation operator improves over retention,
whereas \emph{which} selects the operator with the largest positive gain. A
zero-gain tie is resolved toward retention; ties among positive-gain operators may
be resolved arbitrarily.
\end{proposition}

The result follows by comparing the three operator gains with retention's zero
gain; a formal proof is provided in the supplementary material.

\paragraph{Interpretation and falsifiable implications.}
Equation~\eqref{eq:decomp} does not guarantee a crossover. As $B$ decreases,
$G_B$ shrinks; consolidation helps only when its coverage effect outweighs any
negative replacement effect. It also predicts that harm should concentrate where
compression reduces fidelity, such as timestamps, corrections, and entity
bindings. The experiments test these directional implications but do not identify
the latent fidelity functions or equate the surrogate with LLM answer accuracy.
Which named operator has the largest gain remains an empirical question.

\subsection{Offline Abstraction-Safety Learning}
\label{sec:oas}

The mechanism indicates that the optimal decision depends on unobservable quantities (the true
fidelity of a not-yet-generated compact record). Offline Abstraction-Safety (OAS)
makes the decision learnable by
recovering supervision \emph{in hindsight} and fitting a lightweight router that acts on
features available \emph{before} generation.

\paragraph{Labeling.}
For each training instance $z=(\mem,C,q,B)$, we evaluate retention and every generative
action under the same retriever, budget, answer model, and grader. This produces
action utilities $u_a(z)$, paired gains from Eq.~\eqref{eq:gain}, and a harm event
when a generated action has lower utility than retention. These outcomes are not semantic ground truth; they are
system-specific supervision aligned with the downstream answer metric. They are
observed only offline and are never used as test-time features.

\paragraph{Observability separation.}
The query and candidate evidence are observed, but the decision is made
\emph{before} the compact record is generated and its answer utility is known.
OAS fits a utility router on pre-generation features and supervises it with outcomes recovered
offline. This is query-time adaptive packing, not query-agnostic memory maintenance.

\paragraph{Online-observable feature map.}
Each observed instance $z$ is represented by an 11-dimensional feature vector
$\feat(z)\in\mathbb{R}^{11}$. Seven entries are scalar features: normalized budget
$B/512$, constituent count, token pressure (candidate-evidence cost divided by
$B$), the budget-fit fraction $f_B$ (the share of candidate evidence
packed by retention), the number of represented sessions, an \emph{inconsistency
proxy} measured by the mean squared distance of note embeddings from their cluster
mean, and mean pairwise cosine similarity. The remaining four entries form a
one-hot encoding of the benchmark query type. The constant 512 is used only to
place the budget feature on a convenient numerical scale. In the controlled oracle-evidence protocol these are
available before generation; candidate-evidence cost and benchmark type would need
to be estimated by a retriever and type classifier in an end-to-end deployment.

\paragraph{Direct utility router and hypothesis class.}
Rather than decompose the decision into a binary when-head and a separately fitted
operator heuristic, OAS directly predicts the utility of every action. For each
$a\in\mathcal A$, it fits a regularized linear response with parameter vector
$\theta_a\in\mathbb R^{11}$ and intercept $b_a\in\mathbb R$:
\begin{equation}
\label{eq:gate}
  \widehat\mu_a(z)=b_a+\langle\theta_a,\feat(z)\rangle,
\end{equation}
and deploys $\widehat\pi(z)=\arg\max_{a\in\mathcal A}\widehat\mu_a(z)$. Thus ``when'' is the
comparison with predicted retention utility and ``which'' is the largest predicted
generated-action utility; both are learned from the same low-capacity model.

\paragraph{Utility regression.}
Offline paired evaluation supplies $u_{a,i}\in[0,1]$ for action $a$ on training
instance $i$. We fit the four responses by ridge regression,
\begin{equation}
\label{eq:erm}
\begin{aligned}
  (\widehat b_a,\widehat\theta_a)
  =\arg\min_{b,\theta}\;&\frac{1}{n}\sum_{i=1}^{n}
  \big(u_{a,i}-b-\langle\theta,\feat(z_i)\rangle\big)^2\\
  &+\lambda\lVert\theta\rVert_2^2 .
\end{aligned}
\end{equation}
Here, $n$ is the number of training instances and $\lambda\geq 0$ is the
regularization strength; the intercept is not penalized. The value of $\lambda$ is selected by inner grouped cross-validation using
training questions only, after which the utility functions are frozen.  This
direct objective respects effect magnitude and avoids treating abundant zero-gain
ties as positive or negative classification examples.

Accurate action-utility estimates are sufficient for accurate selection: if every
predicted utility is within $\epsilon$ of its conditional target, plug-in
maximization loses at most $2\epsilon$ in conditional utility
(see the supplementary material). This standard
comparison does not require ridge regression to be correctly specified. We use
the linear form as a regularized, data-efficient inductive bias, not as an
assumption that the true decision boundary is linear. A one-hidden-layer MLP is
evaluated under the same grouped protocol to test whether additional nonlinear
capacity improves utility estimation. The experiments and supplementary material report that it
does not improve reliably over ridge.

\paragraph{Risk-controlled operating threshold.}
Let $o^+(z)=\arg\max_{o\in\mathcal O_{\rm gen}}\widehat\mu_o(z)$ be the generated
action with the highest predicted utility, and define its predicted advantage
$s(z)=\widehat\mu_{o^+(z)}(z)-\widehat\mu_{\ret}(z)$.  Given the frozen utility
router, a separate operating threshold $\tau_{\safe,B}$ is chosen for each budget on a held-out
validation split. Candidate thresholds are the observed validation advantages,
together with an abstain-all option. Calibration first restricts attention to
thresholds whose empirical frequency of harmful
generated selections satisfies a specified harm constraint, and then selects the
threshold with the highest validation accuracy. We require zero observed harmful
selections on the calibration split. The abstain-all option guarantees feasibility. The unconstrained router
selects the action with the largest predicted utility; the calibrated variant
selects $o^+(z)$ only when its predicted advantage reaches the budget-specific
threshold. Otherwise, it retains the raw notes. Because the candidate thresholds
form a nested acceptance family, calibration requires only a sweep over validation
scores. Ties in validation accuracy are resolved toward the more conservative
threshold. This is a conservative operating rule, not a population safety
certificate; the experiments therefore report realized help and harm on disjoint
questions.

\section{Experiments}
\label{sec:experiments}

We evaluate every operator across all budgets on LongMemEval and on a fixed,
outcome-blind, evidence-length-stratified LoCoMo sample. The two datasets have
different evidence-length distributions, allowing us to test whether the
consolidation boundary follows observable budget pressure rather than a fixed token
value.

Our experiments answer three questions. \textbf{(Q1)} Does the value of
consolidation reverse sign with the memory budget? \textbf{(Q2)} Does the best
generative operator (\textsc{Merge}/\textsc{Abstract}/\textsc{Rewrite}) depend on
the budget? \textbf{(Q3)} Can a shared lightweight scorer learn the joint decision of
\emph{when} to consolidate and \emph{which} operator from pre-generation features alone?
Q1 tests the budget-dependent sign predicted by
Proposition~\ref{prop:decomp}; Q2 tests its operator-indexed comparison; a
fit-stratified diagnostic probes the coverage and replacement effects; and Q3
evaluates the learned approximation to Eq.~\eqref{eq:joint-decision}.

\paragraph{Setup.}
LongMemEval~\citep{wu2025longmemeval} is the primary benchmark. Its pool is split
into fit, development, and test roles before modeling, and routing further
separates utility fitting, threshold calibration, and testing by question. In the
controlled evaluation, every action receives the same gold evidence and differs
only in its budgeted representation. Retention packs complete raw items, whereas
\textsc{Abstract}, \textsc{Merge}, and \textsc{Rewrite} produce the representations
defined in the problem formulation. This pairing isolates the when--which
decision from evidence discovery.

LongMemEval uses budgets of 32, 64, 128, and 256 tokens. A single evidence turn has
a median cost of 59 tokens, while retention fully fits all required evidence for
$1.3\%$, $9.3\%$, $52.0\%$, and $85.3\%$ of the test questions at these budgets.
We define budget regimes by this full-fit rate rather than by the absolute token
value: 32 and 64 tokens are severely tight and tight, 128 tokens is the transition
regime, and 256 tokens is loose. LoCoMo
\citep{maharana2024locomo} provides an outcome-blind, evidence-length-stratified
replication with shorter evidence and correspondingly smaller budgets of 16, 32,
64, and 128 tokens. Every generated representation is deterministically capped at
its target budget. Full evidence distributions and implementation details are in
the supplementary material.

Answers are generated and graded with DeepSeek-V3.2 under a common prompt,
with three paired realizations on LongMemEval. Statistical comparisons average
repeats within question and use question-level paired bootstrap intervals and
randomization tests. A GLM-5.2 replication tests answer-model sensitivity. An
additional full-history evaluation supplies both methods with the same candidates
from BM25, a sparse lexical retriever that scores term matches with adjustments for
term rarity and turn length \citep{robertson2009bm25}. These robustness protocols
are detailed in the supplementary material.

\subsection{The Budget Crossover (Q1)}
\label{sec:crossover}

Figure~\ref{fig:crossover} reports the budget sweep. Retention accuracy decreases
from $74.5\%$ at the loose budget to $3.9\%$ at the severely tight budget as fewer
raw evidence turns fit. Every consolidation operator has a higher accuracy point
estimate than retention in the tight and transition regimes, whereas retention has
the highest accuracy at the loose budget. For the representative
\textsc{Abstract} comparison, the tight-budget absolute accuracy gain is $48.0\%$,
with a bootstrap interval from $37.3\%$ to $58.2\%$; its loose-budget difference
is $-8.0\%$, with an interval from $-16.9\%$ to $0.5\%$. The corresponding loose-budget interval for
\textsc{Merge} excludes zero in the negative direction
(supplementary paired inference). The reversal therefore supports the directional
prediction of Proposition~\ref{prop:decomp}: reducing the budget contracts the
region grounded by raw retention and increases the potential coverage effect.
The sweep does not identify the latent terms separately; that implication is tested
by the fit-stratified diagnostic below.

\begin{figure}[t]
\centering
\includegraphics[width=0.74\linewidth]{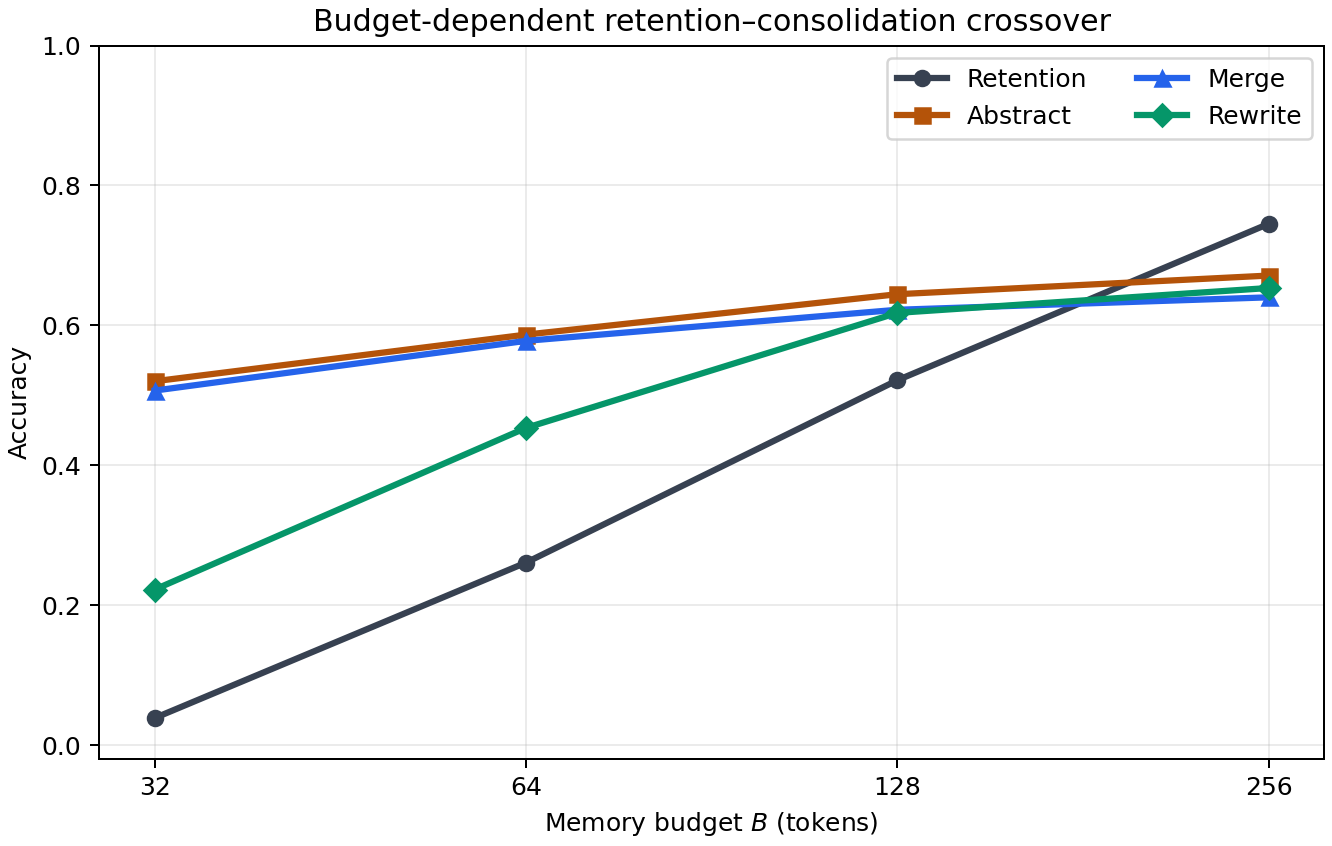}
\caption{\textbf{The budget crossover across consolidation operators.} Mean
accuracy of retention, \textsc{Abstract}, \textsc{Merge}, and \textsc{Rewrite}
on the LongMemEval test partition with three label realizations. At 32 tokens,
\textsc{Abstract} reaches $52.0\%$ accuracy versus $4.0\%$ for its paired
retention baseline, an absolute gain of $48.0\%$. Every consolidation operator
leads retention in the tight regimes, and the advantage narrows in the 128-token
transition regime. At the loose 256-token budget, retention exceeds every
consolidation operator. Thus consolidation is favored when raw evidence does not
fit, whereas retention is favored when most evidence already fits. Question-level paired uncertainty is reported in
the supplementary material. The plotted retention curve averages all paired
retention realizations; operator effects and confidence intervals use each
operator's within-run retention counterpart.}
\label{fig:crossover}
\end{figure}

\paragraph{Cross-model replication.}
The complete retain-versus-abstract sweep with GLM-5.2 reproduces the same answer
to \emph{when}: abstraction improves accuracy in the severely tight regime and is
worse than retention in the loose regime. Full results are reported in
the supplementary material.

\paragraph{Full-history retrieval evaluation.}
With identical lexically retrieved candidates, abstraction again improves over
retention in the tight regimes, while its advantage disappears as the budget
loosens. This auxiliary result shows that the budget interaction survives retrieval
noise; it does not establish superiority over optimized end-to-end memory systems.
Full results and paired uncertainty are in
the supplementary material.

\begin{table}[t]
\centering
\scriptsize
\begin{tabular}{@{}cccc@{}}
\toprule
Operator & $B$ & Retention fit & Gain (help/harm; $n$) \\
\midrule
\multirow{4}{4.8em}{\centering \textsc{Abstract}} & 128 & Incomplete & $+.343$ (.38/.04; 36) \\
 & 128 & Complete & $-.043$ (.06/.10; 39) \\
 & 256 & Incomplete & $-.061$ (.21/.27; 11) \\
 & 256 & Complete & $-.083$ (.07/.16; 64) \\
\midrule
\multirow{4}{4.8em}{\centering \textsc{Merge}} & 128 & Incomplete & $+.324$ (.41/.08; 36) \\
 & 128 & Complete & $-.145$ (.06/.21; 39) \\
 & 256 & Incomplete & $+.152$ (.27/.12; 11) \\
 & 256 & Complete & $-.151$ (.04/.19; 64) \\
\midrule
\multirow{4}{4.8em}{\centering \textsc{Rewrite}} & 128 & Incomplete & $+.306$ (.39/.08; 36) \\
 & 128 & Complete & $-.094$ (.08/.17; 39) \\
 & 256 & Incomplete & $-.061$ (.18/.24; 11) \\
 & 256 & Complete & $-.089$ (.07/.16; 64) \\
\bottomrule
\end{tabular}
 \caption{\textbf{Fit-stratified mechanism diagnostic.} Paired gain according to
whether retention packs all candidate evidence; parentheses report helped and
harmed rates and stratum size. The table uses the transition and loose budgets,
where both fit strata are sufficiently populated for comparison. At 32 and 64
tokens, complete fit occurs for only one and seven questions, respectively; their
aggregate tight-budget results appear in Figure~\ref{fig:crossover}.}
\label{tab:mechanismfit}
\end{table}

\subsection{Which Operator? Budget-Dependent Tradeoffs (Q2)}
\label{sec:operator2d}

The operator leader changes with budget. In both tight regimes,
\textsc{Abstract} ranks first, followed by \textsc{Merge}, \textsc{Rewrite}, and
retention; under the loose budget, retention ranks first and every generated
operator has a negative paired gain. The two cross-note operators outperform local
\textsc{Rewrite} under compression, but their direct difference is not
significant. Thus \emph{which} is not a universal named operator: cross-note fusion
is preferred for fragmented evidence, while the choice between abstraction and
merging remains instance dependent. Full pairwise tests and per-type results are
in the supplementary material.

\subsection{Mechanism Diagnostic: Coverage and Replacement}
\label{sec:mechanism-diagnostic}
We test the surrogate's observable implication by stratifying outcomes on the
retention fit fraction $f_B$. Table~\ref{tab:mechanismfit} shows that every
operator gains when some evidence does not fit but loses when all evidence fits in
the transition regime; the full-fit loss persists under the loose budget. This
coverage--replacement interaction supports Proposition~\ref{prop:decomp} without
claiming to identify its latent fidelity terms.

\subsection{A Unified Gate for When and Which (Q3)}
\label{sec:unified}

OAS is evaluated by grouped five-fold cross-fitting: utility fitting, threshold
calibration, and testing use disjoint questions, and all budgets of a question
remain in one fold. Table~\ref{tab:unified} shows significant gains over retention
in the three stressed regimes and no difference under the loose budget. Selected
actions shift from cross-note operators toward retention as evidence fit grows,
thereby realizing both \emph{when} and \emph{which}. Calibration reduces harmful
replacements at some utility cost; its empirical constraint is not a population
safety certificate.

\begin{table}[t]
\centering
\scriptsize
\begin{tabular}{@{}ccccc@{}}
\toprule
\multicolumn{1}{c}{Policy} & $B$ & Accuracy & Gain [95\% CI] & Harm \\
\midrule
\multirow{4}{4.8em}{\centering Direct\\OAS} & 32 & .545 & $+.507$ [.404,.607] & 3 \\
 & 64 & .618 & $+.357$ [.227,.483] & 12 \\
 & 128 & .640 & $+.119$ [.015,.224] & 14 \\
 & 256 & .745 & $.000$ [$-.061$,.064] & 10 \\
\midrule
\multirow{4}{4.8em}{\centering Calibrated\\OAS} & 32 & .487 & $+.449$ [.345,.551] & 1 \\
 & 64 & .403 & $+.142$ [.056,.234] & 3 \\
 & 128 & .613 & $+.092$ [.034,.160] & 2 \\
 & 256 & .742 & $-.003$ [$-.009$,.000] & 1 \\
\bottomrule
\end{tabular}
 \caption{\textbf{Joint when--which routing.} Direct OAS selects the action with
the largest predicted utility. Calibrated OAS applies a budget-specific harm
threshold fitted on disjoint questions and otherwise retains the raw evidence.
Gain is relative to retention; Harm counts lower-utility generated selections.
All results use grouped five-fold cross-fitting.}
\label{tab:unified}
\end{table}

\subsection{Cross-Dataset Replication on LoCoMo}
\label{sec:locomo}
LoCoMo replicates both conclusions. In its severely tight regime,
\textsc{Abstract} and \textsc{Merge} attain $48.0\%$ accuracy, compared with
$12.9\%$ for retention and $26.7\%$ for \textsc{Rewrite}; their paired gains remain
significant in both tight regimes. The advantage disappears once 64 tokens
accommodate the evidence for approximately three quarters of the sample. Because
LoCoMo evidence is shorter, this earlier crossover supports relative fit pressure,
as predicted by Proposition~\ref{prop:decomp}, rather than an absolute token
threshold. It also replicates the \emph{which} result: cross-note synthesis is
preferred under compression, without establishing universal dominance between
\textsc{Abstract} and \textsc{Merge}. Complete operator results, paired inference,
router controls, and the MLP comparison are in the supplementary material.

\section{Conclusion}
\label{sec:conclusion}

We formulate memory management as a joint, budget-dependent when--which decision.
Its coverage and replacement effects explain the consistent pattern across
LongMemEval and LoCoMo: consolidation is preferred while relevant raw evidence
does not fit, retention becomes preferable after it fits, and cross-note synthesis
is more effective than local rewriting under compression. OAS operationalizes
this mechanism through lightweight multi-action utility learning. The results
establish relative budget pressure, rather than a universal operator or token
threshold, as the central determinant of consolidation.

\paragraph{Limitations and future work.}
The controlled study isolates representation from retrieval; future work can
jointly learn retrieval, consolidation, and packing in persistent memory.
 
\bibliography{references}

\clearpage
\raggedbottom
\twocolumn[{\centering\LARGE\bfseries Supplementary Material\par\vspace{1.2em}}]
\appendix

\section{Scope of the When--Which Formalization}
\label{app:theory}

The decomposition in Proposition~\ref{prop:decomp} is exact only for the surrogate
$\mathcal U_B$ under Assumption~\ref{ass:uncovered}. It is intentionally not stated
as an identity for binary LLM-judge accuracy. This distinction makes the mechanism
falsifiable: the budget sweep can agree or disagree with its directional prediction
without being used to define the latent fidelity functions. Proposition~\ref{prop:joint}
then formalizes the joint action implied by those operator-specific values. It does
not impose a universal semantic ranking among \textsc{Merge}, \textsc{Abstract},
and \textsc{Rewrite}; the comparison across operators and budgets therefore
determines the ranking empirically.

\begin{proof}[Proof of Proposition~\ref{prop:decomp}]
Fix an operator $o$ and write
$d_o(q)=s_B(q,T_{C,o,B}(\mem))-s_B(q,\mem)$. By
Assumption~\ref{ass:uncovered}, $d_o(q)$ is zero outside $H_{B,o}$ and on the
unchanged part of $H_{B,o}\cap G_B$. On $X_{B,o}$, retention contributes zero and
consolidation contributes $\varphi_{\tilde m_{o,B}}(q)$. On $J_{B,o}$, the
contribution changes from $\varphi_{m_{\rm raw}(q)}(q)$ to
$\varphi_{\tilde m_{o,B}}(q)$. Therefore,
\begin{align}
d_o(q)
&=\mathbf 1\{q\in X_{B,o}\}\varphi_{\tilde m_{o,B}}(q) \\
&\quad+\mathbf 1\{q\in J_{B,o}\}
\bigl(\varphi_{\tilde m_{o,B}}(q)-\varphi_{m_{\rm raw}(q)}(q)\bigr).
\end{align}
Taking expectation under $F$ gives $C_{B,o}+R_{B,o}$ by conditional
expectation. The stated sign conditions follow directly from the two terms.
\end{proof}

\begin{proof}[Proof of Proposition~\ref{prop:joint}]
Retention has surrogate gain zero. Let
$\delta^\star=\max_{o\in\mathcal O_{\rm gen}}\Delta\mathcal U_{B,o}$.
If $\delta^\star\leq 0$, no generative action improves on retention, so the stated
tie rule selects retention. If $\delta^\star>0$, any operator attaining
$\delta^\star$ has greater surrogate utility than retention and no smaller utility
than any other generative action. Selecting such an operator is therefore optimal.
\end{proof}

\subsection{Plug-in Utility Bound}
\label{app:plugin-bound}

The following standard bound explains why action-utility regression is aligned
with the finite when--which decision. It is conditional on estimation accuracy and
does not assert that OAS attains a particular error; that question is evaluated by
held-out experiments.

\begin{proposition}[Utility estimation controls decision regret]
For an instance $z$, suppose every predicted action utility satisfies
$|\widehat\mu_a(z)-\mu_a(z)|\leq\epsilon(z)$. Let $a^\star(z)$ maximize the true
conditional utility and let $\widehat\pi(z)$ maximize the predicted utility. Then
the conditional utility lost by the learned when--which decision is at most
$2\epsilon(z)$:
\begin{equation}
  \mu_{a^\star(z)}(z)-\mu_{\widehat\pi(z)}(z)\leq 2\epsilon(z).
\end{equation}
\end{proposition}

\begin{proof}
The true utility of $a^\star(z)$ is at most its predicted utility plus
$\epsilon(z)$. Predicted maximization makes this no greater than the predicted
utility of $\widehat\pi(z)$ plus $\epsilon(z)$, which is at most the true utility
of $\widehat\pi(z)$ plus $2\epsilon(z)$.
\end{proof}

\section{Reproducibility and Experimental Detail}
\label{app:repro}

This section first reports the paired inference underlying the main operator
comparisons, then collects auxiliary robustness analyses, and finally documents
the shared evaluation protocol and router-capacity checks. All tables use the same
action definitions and budget caps as Sections~\ref{sec:decision}
and~\ref{sec:experiments}.

\subsection{Primary Paired Inference}

Tables~\ref{tab:operatorstats} and~\ref{tab:locomostats} report the
question-level paired inference underlying the main when--which conclusions on
LongMemEval and LoCoMo. The first tests each operator against retention across
budgets; the second verifies the corresponding budget-dependent pattern on the
shorter-evidence benchmark.

\begin{table}[H]
\centering
\small
\begin{tabular}{cccc}
\toprule
Operator & Budget & Paired gain [95\% CI] & one-sided $p$ \\
\midrule
\textsc{Abstract} & 32 & $+0.480$ $[+0.373,+0.582]$ & below $.001$ \\
\textsc{Abstract} & 64 & $+0.333$ $[+0.204,+0.458]$ & below $.001$ \\
\textsc{Abstract} & 128 & $+0.142$ $[+0.044,+0.244]$ & $0.005$ \\
\textsc{Abstract} & 256 & $-0.080$ $[-0.169,+0.005]$ & $0.969$ \\
\midrule
\textsc{Merge} & 32 & $+0.458$ $[+0.347,+0.569]$ & below $.001$ \\
\textsc{Merge} & 64 & $+0.316$ $[+0.191,+0.436]$ & below $.001$ \\
\textsc{Merge} & 128 & $+0.080$ $[-0.036,+0.196]$ & $0.108$ \\
\textsc{Merge} & 256 & $-0.107$ $[-0.204,-0.013]$ & $0.989$ \\
\midrule
\textsc{Rewrite} & 32 & $+0.196$ $[+0.116,+0.280]$ & below $.001$ \\
\textsc{Rewrite} & 64 & $+0.187$ $[+0.067,+0.307]$ & $0.002$ \\
\textsc{Rewrite} & 128 & $+0.098$ $[-0.018,+0.218]$ & $0.061$ \\
\textsc{Rewrite} & 256 & $-0.084$ $[-0.178,+0.009]$ & $0.967$ \\
\bottomrule
\end{tabular}
 \caption{\textbf{Question-level inference for operator performance across budgets.} Each interval
is a 10,000-replicate paired bootstrap over question IDs after averaging the three
label realizations within question. The reported $p$ tests the one-sided beneficial
alternative; for negative effects at 256 tokens, values near one reject benefit rather
than directly reporting the reverse one-sided test. The 32-token Merge cell is the
strict rerun with both a 32-token API ceiling and a lexical hard cap.}
\label{tab:operatorstats}
\end{table}

\begin{table}[t]
\centering
\small
\resizebox{\columnwidth}{!}{%
\begin{tabular}{ccccc}
\toprule
Budget & Operator & Paired gain & 95\% bootstrap interval & Helped/Harmed \\
\midrule
16 & \textsc{Abstract} & $+0.360$ & $[+0.227,+0.480]$ & 30/3 \\
16 & \textsc{Merge}    & $+0.347$ & $[+0.227,+0.467]$ & 29/3 \\
16 & \textsc{Rewrite}  & $+0.133$ & $[+0.013,+0.253]$ & 16/6 \\
32 & \textsc{Abstract} & $+0.173$ & $[+0.040,+0.307]$ & 20/7 \\
32 & \textsc{Merge}    & $+0.253$ & $[+0.120,+0.387]$ & 24/5 \\
32 & \textsc{Rewrite}  & $+0.133$ & $[+0.000,+0.267]$ & 19/9 \\
64 & \textsc{Abstract} & $+0.000$ & $[-0.107,+0.107]$ & 8/8 \\
64 & \textsc{Merge}    & $-0.013$ & $[-0.107,+0.080]$ & 6/7 \\
64 & \textsc{Rewrite}  & $-0.027$ & $[-0.120,+0.067]$ & 6/8 \\
128 & \textsc{Abstract} & $-0.027$ & $[-0.093,+0.053]$ & 3/5 \\
128 & \textsc{Merge}    & $-0.040$ & $[-0.120,+0.040]$ & 3/6 \\
128 & \textsc{Rewrite}  & $+0.000$ & $[-0.093,+0.093]$ & 6/6 \\
\bottomrule
\end{tabular}
}
 \caption{\textbf{Question-level inference for LoCoMo operator performance across budgets.}
The evaluation sample is drawn before outcome generation with fixed seed
20260716, stratified proportionally over four gold-evidence-cost strata. Each row
uses the operator-paired retention realization and 10,000 paired bootstrap samples
over question IDs. Helped and Harmed count questions on which the generated action
is better or worse than retention.}
\label{tab:locomostats}
\end{table}

\subsection{Auxiliary Robustness Tables}
\label{app:robustness-tables}

The following analyses test whether the main when--which conclusions depend on the
answer model, controlled evidence supply, question type, or router class. They are
reported separately because they support robustness rather than define the primary
mechanism evidence.

\begin{table}[t]
\centering
\small
\resizebox{\columnwidth}{!}{%
\begin{tabular}{cccccc}
\toprule
Model & Budget & Retention & Abstract & Paired gain [95\% CI] & Helped/Harmed \\
\midrule
DeepSeek-V3.2 & 32 & 0.040 & 0.520 & +0.480 [+0.373,+0.582] & 49/2 \\
 & 64 & 0.253 & 0.587 & +0.333 [+0.204,+0.458] & 37/10 \\
 & 128 & 0.502 & 0.644 & +0.142 [+0.044,+0.244] & 22/11 \\
 & 256 & 0.751 & 0.671 & -0.080 [-0.169,+0.005] & 8/18 \\
\midrule
GLM-5.2 & 32 & 0.027 & 0.507 & +0.480 [+0.360,+0.600] & 37/1 \\
 & 64 & 0.267 & 0.627 & +0.360 [+0.227,+0.493] & 32/5 \\
 & 128 & 0.453 & 0.707 & +0.253 [+0.120,+0.387] & 24/5 \\
 & 256 & 0.707 & 0.560 & -0.147 [-0.267,-0.027] & 7/18 \\
\bottomrule
\end{tabular}
}
 \caption{\textbf{Cross-model budget crossover} on the same LongMemEval questions.
DeepSeek-V3.2 averages three paired realizations per question; GLM-5.2 uses one
complete realization scored by the fixed DeepSeek-V3.2 judge. Intervals are 10,000 question-level paired bootstrap
replicates; Helped and Harmed denote questions improved or worsened after
averaging available reruns.}
\label{tab:crossmodel}
\end{table}

\begin{table}[t]
\centering
\small
\resizebox{\columnwidth}{!}{%
\begin{tabular}{cccccc}
\toprule
Budget & Retention & Abstract & Paired gain [95\% CI] & Helped/Harmed & Binary oracle \\
\midrule
32 tokens & 0.040 & 0.400 & $+0.360$ $[+0.240,+0.480]$ & 28/1 & 0.413 \\
64 tokens & 0.240 & 0.413 & $+0.173$ $[+0.040,+0.307]$ & 20/7 & 0.507 \\
128 tokens & 0.427 & 0.453 & $+0.027$ $[-0.107,+0.160]$ & 13/11 & 0.600 \\
256 tokens & 0.480 & 0.453 & $-0.027$ $[-0.147,+0.093]$ & 9/11 & 0.600 \\
\bottomrule
\end{tabular}
}
 \caption{\textbf{Full-history lexical-retrieval evaluation} on the complete LongMemEval test
split (one complete realization). Both methods see
the identical retrieved candidate set (top 20, capped at 2,048 tokens); retention
packs candidates in rank order, while \textsc{Abstract} compresses the candidate set into the
target budget. Intervals use 10,000 paired question bootstrap samples; Helped and
Harmed count improved and worsened questions. The binary oracle is descriptive
hindsight, not a trained router.}
\label{tab:retrieved}
\end{table}

\begin{table}[t]
\centering
\small
\resizebox{\columnwidth}{!}{%
\begin{tabular}{ccccc}
\toprule
 & & \multicolumn{3}{c}{Paired gain over retention (helped/harmed, mean per run)} \\
\cmidrule(lr){3-5}
Question type & $n$ & \textsc{Abstract} & \textsc{Merge} & \textsc{Rewrite} \\
\midrule
Single-session & 9 & \textbf{$+0.037$ (0.7/0.3)} & $-0.037$ (0.0/0.3) & $-0.037$ (1.0/1.3) \\
Multi-session & 13 & \textbf{$+0.051$ (2.7/2.0)} & \textbf{$+0.179$ (2.3/0.0)} & $-0.077$ (0.7/1.7) \\
Temporal & 13 & $-0.154$ (1.7/3.7) & $-0.051$ (2.0/2.7) & \textbf{$+0.051$ (2.3/1.7)} \\
Knowledge-update & 40 & $-0.125$ (2.0/7.0) & $-0.233$ (1.0/10.3) & $-0.142$ (2.3/8.0) \\
\midrule
\emph{All} & 75 & $-0.080$ (7.0/13.0) & $-0.107$ (5.3/13.3) & $-0.084$ (6.3/12.7) \\
\bottomrule
\end{tabular}
}
 \caption{\textbf{Operator performance by question type at the loose budget}
(256 tokens, mean over three label realizations). Each cell is the paired gain over
retention with helped and harmed counts per realization. Harm concentrates in
knowledge-update and temporal questions, while multi-session questions provide the
main exception. These estimates are descriptive complements to the aggregate
when--which comparisons.}
\label{tab:pertype}
\end{table}

\begin{table}[t]
\centering
\small
\resizebox{\columnwidth}{!}{%
\begin{tabular}{ccccc}
\toprule
Policy & 32 tokens & 64 tokens & 128 tokens & 256 tokens \\
\midrule
Retention & 0.039 (0) & 0.261 (0) & 0.521 (0) & 0.745 (0) \\
Training-selected fixed operator & 0.493 (5) & 0.569 (13) & 0.631 (14) & 0.631 (26) \\
Evidence-fit rule & 0.493 (5) & 0.563 (12) & 0.681 (3) & 0.729 (4) \\
$k$-NN utility router & 0.487 (4) & 0.556 (14) & 0.618 (13) & 0.747 (6) \\
Direct utility (ours) & 0.545 (3) & 0.618 (12) & 0.640 (14) & 0.745 (10) \\
Calibrated-safe (ours) & 0.487 (1) & 0.403 (3) & 0.613 (2) & 0.742 (1) \\
\bottomrule
\end{tabular}
}
 \caption{\textbf{Routing baselines under the same grouped cross-fitting protocol.}
Cells report accuracy with the number of harmful generated choices in parentheses.
All learned hyperparameters use only outer-training questions; only the frozen
safe router's threshold uses the calibration fold.}
\label{tab:routerbaselines}
\end{table}

\subsection{Shared Protocol and Statistical Units}

\paragraph{Fixed configuration.} Primary operator and router runs use
DeepSeek-V3.2 (temperature $0$) behind an OpenAI-compatible endpoint and
text-embedding-3-small for the geometry features. Retriever packing,
answer prompt, and grader prompt are held fixed across every policy, budget, and
operator; only the memory state varies. The cross-model table reruns the complete
retain-versus-abstract comparison with GLM-5.2, the same prompts, one paired
realization per question, and the fixed primary DeepSeek-V3.2 judge. We partition the $500$-item
LongMemEval\_S into disjoint train ($350$), dev ($75$), and test ($75$) splits
before the reported operator comparisons. The budget and operator sweeps cover the
entire test partition; no question is filtered by outcome, evidence length, or operator
success. Question-level paired bootstrap and sign-flip inference are computed from
the cached outcomes. The
utility-router table is a grouped five-fold cross-fitted estimate on the controlled
evaluation matrix: for each rotation, three folds train, one disjoint fold calibrates thresholds,
and the last fold tests; all budgets of a question remain in one fold.

\paragraph{Independent units and claim precision.}
The four budgets and three label realizations are repeated measurements, not
additional independent samples. We first average realizations within question and
then perform inference across the held-out question IDs. Consequently, the paper uses
the evaluation to support effects whose question-bootstrap intervals exclude zero, such
as the tight-budget gain and the advantage of cross-note fusion over local
\textsc{Rewrite}. It reports the point-estimate ordering of \textsc{Abstract} and
\textsc{Merge} but does not claim dominance because their paired intervals include
zero. The multilayer perceptron (MLP) capacity check follows the same rule. This explicit hierarchy avoids
pseudoreplication while retaining every available test question.

\paragraph{Budgets, operators, and label realizations.} For LongMemEval, we evaluate
budgets of 32, 64, 128, and 256 tokens.
\textsc{Abstract} emits a higher-level summary, \textsc{Merge} emits a compact
factual union, and \textsc{Rewrite} reformulates notes separately before packing
them cheapest-first. Every resulting representation is capped at the target
budget. Because the
answerer and grader are stochastic across reruns, the budget crossover and the
comparison of operators across budgets relabel the test split three times, and the
breakdown by question type and operator uses the same three realizations, all reported with
across-run means. Operator gains use each file's within-run retention pair. For the
unified action comparison, each action's utility is its three-run mean and retention is
the mean over all nine paired reruns, avoiding the invalid combination of one
operator file with another file's stochastic retention label.

\begin{table}[t]
\centering
\small
\resizebox{\columnwidth}{!}{%
\begin{tabular}{cccccc}
\toprule
$B$ & Mean fit & Full (\%) & Zero (\%) & Partial (\%) & Regime \\
\midrule
32  & .065 & 1.3  & 78.7 & 20.0 & Severely tight \\
64  & .424 & 9.3  & 10.7 & 80.0 & Tight \\
128 & .718 & 52.0 & 1.3  & 46.7 & Transition \\
256 & .911 & 85.3 & 0.0  & 14.7 & Loose \\
\bottomrule
\end{tabular}
}
 \caption{\textbf{Budget pressure on the controlled LongMemEval test split.}
Fit fraction is the share of gold evidence turns packed by cheapest-first
retention. ``Full,'' ``zero,'' and ``partial'' report the percentage of questions for which all, none, or some of the
gold turns fit. Raw-turn cost uses the same alphanumeric-token count as dataset
construction.}
\label{tab:budgetpressure}
\end{table}

\begin{table}[t]
\centering
\scriptsize
\resizebox{\columnwidth}{!}{%
\begin{tabular}{cccc}
\toprule
$B$ & Merge $-$ Abstract & Abstract $-$ Rewrite & Merge $-$ Rewrite \\
\midrule
32  & $-.013$ [$-.093,+.067$] & $+.298$ [$+.196,+.396$] & $+.284$ [$+.173,+.396$] \\
64  & $-.009$ [$-.080,+.058$] & $+.133$ [$+.027,+.240$] & $+.124$ [$+.018,+.227$] \\
128 & $-.022$ [$-.093,+.049$] & $+.027$ [$-.062,+.120$] & $+.004$ [$-.089,+.098$] \\
256 & $-.031$ [$-.111,+.049$] & $+.018$ [$-.062,+.098$] & $-.013$ [$-.107,+.080$] \\
\bottomrule
\end{tabular}
}
 \caption{\textbf{Direct operator-to-operator comparisons.} Entries are paired
differences in answer accuracy after averaging the three realizations within
question, with 95\% question-bootstrap intervals. Positive values favor the
row's first named operator. These comparisons reuse cached outcomes and require
no additional model calls beyond the strict 32-token Merge rerun.}
\label{tab:operatorpairwise}
\end{table}

\paragraph{Unified router.} The low-capacity router fits one ridge utility response
per action and selects the action with the highest predicted utility. Ridge regularization is selected by
three-fold grouped CV entirely inside each outer training partition. The resulting
generated-action selector and its predicted advantage over retention are then
frozen; only a per-budget threshold is selected on the disjoint calibration fold.
No test question informs the utility model, hyperparameter, or threshold. This
estimates within-benchmark generalization; it is not a population safety
certificate.

\subsection{Nonlinear Router Capacity Check}
\label{app:mlpcapacity}

We test whether the linear utility heads merely underfit a nonlinear action
boundary. The comparison uses a one-hidden-layer ReLU MLP with four sigmoid
utility outputs as a minimal nonlinear control. Hidden widths of 4, 8, and 16,
weight-decay values of 0.001, 0.01, and 0.1, and the early-stopping epoch are selected by
three-fold question-grouped CV wholly inside each outer-training partition. It is
then refit on those training questions and evaluated on the untouched outer test
fold; the outer calibration fold is unused. Thus the 300 budget rows are never
treated as 300 independently split samples.

\begin{table}[H]
\centering
\small
\resizebox{\columnwidth}{!}{%
\begin{tabular}{cccc}
\toprule
Budget & Ridge accuracy & MLP accuracy & Difference from ridge [95\% CI] \\
\midrule
32 & 0.545 & 0.529 & $-0.016$ $[-0.049,+0.007]$ \\
64 & 0.618 & 0.604 & $-0.013$ $[-0.040,+0.004]$ \\
128 & 0.640 & 0.644 & $+0.004$ $[-0.027,+0.036]$ \\
256 & 0.745 & 0.708 & $-0.037$ $[-0.102,+0.024]$ \\
\bottomrule
\end{tabular}
}
 \caption{\textbf{Nonlinear capacity check.} Out-of-fold action accuracy under the
same five-fold question-grouped protocol. Intervals are paired question-bootstrap
intervals for the difference between MLP and ridge. None excludes zero, and the
two-sided sign-flip tests do not identify a significant difference at any budget.}
\label{tab:mlpcapacity}
\end{table}

The MLP has absolute accuracy $1.6\%$ and $1.3\%$ lower than ridge at 32 and 64
tokens, $0.4\%$ higher at 128 tokens, and $3.7\%$ lower at 256 tokens; no difference is statistically
distinguishable. We therefore retain ridge as the primary lightweight learner.
This choice is empirical and variance-aware, not an attempt to avoid neural
methods. The purpose is to test whether basic nonlinearity changes the result,
rather than to conduct a separate architecture search. Under the common grouped
protocol, added nonlinear capacity does not provide reliable additional utility
and is less attractive for harm calibration.

\subsection{Full-History Retrieval Stress Test}

\paragraph{Controlled setting.} Throughout, ``retention,'' each operator, and the
oracle gate receive the \emph{same} gold evidence per question; the study measures
the consolidation decision in isolation. The separate retrieval diagnostic starts
from every test question's full chronological memory pool. It uses a local,
deterministic implementation of the lexical retriever defined in the main setup, with a term-frequency saturation parameter of
1.2 and a length-normalization parameter of 0.75, retrieves the 20 highest-ranked
dialogue turns, and greedily applies a common 2,048-token candidate cap. From this
same capped ranked list, retention packs complete turns into $B$, while abstraction
generates one record capped to $B$; answer and judge prompts are shared. The
complete split contains 9 single-session, 13 multi-session, 13 temporal, and 40
knowledge-update questions. We report one realization at
the four budgets of 32, 64, 128, and 256 tokens and bootstrap paired question
differences 10,000 times.
This is a retrieval stress test, not a tuned retriever or full-system leaderboard
comparison.

\subsection{Independently Split Full-History Four-Action Evaluation}
\label{app:routeA}
This stricter transfer test uses a type-stratified split of previously unused
LongMemEval training questions into disjoint fit, calibration, and gate
partitions. The lexical retriever supplies the same top 20 turns to retention,
\textsc{Abstract}, \textsc{Merge}, and \textsc{Rewrite} under a common 2,048-token
candidate cap. Generation, answering, and fixed judging use DeepSeek-V3.2 at
temperature zero. Router features exclude gold-evidence indicators, and all model
choices and calibration thresholds are frozen before the gate outcomes are
aggregated. This transfer test uses its own prespecified harm cap, so its safe
threshold is distinct from the zero-observed-harm calibration used in the
controlled cross-fitting experiment.

Table~\ref{tab:routeAgate} reports one complete realization. Fixed
\textsc{Merge}, tied with the per-budget policy selected on fit data, obtains the
highest macro accuracy of $45.7\%$. Direct ridge reaches $44.7\%$, while the frozen
safe router reaches $39.3\%$ and reduces harmful generated choices from nine to
six relative to fixed \textsc{Merge}. The hindsight action oracle reaches
$50.3\%$, so heterogeneous action value exists, but the safe pre-generation router
does not rank it reliably on the independent full-history split. This result is
the basis for the paper's restricted routing claim: the controlled evaluation
supports learnable when--which structure, whereas transferable end-to-end routing
is not established.

\begin{table}[H]
\centering
\small
\resizebox{\columnwidth}{!}{%
\begin{tabular}{cccccc}
\toprule
Policy & 32 tokens & 64 tokens & 128 tokens & 256 tokens & Macro \\
\midrule
Retention & .173 (0) & .267 (0) & .387 (0) & .480 (0) & .327 \\
Fixed abstract & .373 (2) & .400 (4) & .480 (1) & .453 (6) & .427 \\
Fixed merge & .387 (2) & .453 (2) & .467 (2) & .520 (3) & \textbf{.457} \\
Fixed rewrite & .320 (2) & .440 (1) & .440 (2) & .480 (3) & .420 \\
Per-budget policy selected on fit data & .387 (2) & .453 (2) & .467 (2) & .520 (3) & \textbf{.457} \\
Direct ridge & .347 (2) & .440 (2) & .480 (1) & .520 (2) & .447 \\
Frozen safe router & .347 (2) & .307 (2) & .440 (0) & .480 (2) & .393 \\
Action oracle & .440 (0) & .520 (0) & .493 (0) & .560 (0) & .503 \\
\bottomrule
\end{tabular}
}
 \caption{\textbf{Independent full-history four-action evaluation} (one label realization).
Cells are accuracy with harmful generated choices
in parentheses. Macro averages the four budget accuracies. All fit choices and
calibration thresholds are frozen before gate aggregation; the action oracle is
descriptive hindsight.}
\label{tab:routeAgate}
\end{table}
 
\end{document}